\documentclass[letterpaper, 10 pt, conference]{ieeeconf}  % Comment this line out if you need a4paper

\IEEEoverridecommandlockouts                              % This command is only needed if 
\usepackage{multirow}
\usepackage{lscape} % for large table
\usepackage{tabularx} % for large table
\usepackage{graphicx} % for \graphicspath
\usepackage{comment}
\usepackage{color}
\usepackage{algorithm}
\usepackage[noend]{algpseudocode} % algorithmic, \State{}, etc
\usepackage{amssymb,amsmath}
\newtheorem{theorem}{Theorem}

\newtheorem{definition}{Definition}

\title{\LARGE \bf
	Subdimensional Expansion for Multi-objective Multi-agent Path Finding
}

\author{Zhongqiang Ren$^{1}$, Sivakumar Rathinam$^{2}$ and Howie Choset$^{1}$% <-this % stops a space
	\thanks{$^{1}$ Zhongqiang Ren and Howie Choset are with Carnegie Mellon University, 5000 Forbes Ave., Pittsburgh, PA 15213, USA. 
	}%
	\thanks{$^{2}$Sivakumar Rathinam is with Texas A\&M University,
		College Station, TX 77843-3123.
	}
}
\begin{document}

\maketitle

\thispagestyle{empty}
\pagestyle{empty}

% \thispagestyle{plain}
% \pagestyle{plain}
% \pagenumbering{arabic}

\begin{abstract}
	Conventional multi-agent path planners typically determine a path that optimizes a single objective, such as path length. Many applications, however, may require multiple objectives, say time-to-completion and fuel use, to be simultaneously optimized in the planning process. Often, these criteria may not be directly compared and sometimes lie in competition with each other. Simply applying standard multi-objective search algorithms to multi-agent path finding  may prove to be inefficient because the size of the space of possible solutions, i.e., the Pareto-optimal set, can grow exponentially with the number of agents. This paper presents an approach that bypasses this so-called curse of dimensionality by leveraging our prior multi-agent work with a framework called subdimensional expansion. One example of subdimensional expansion, when applied to A$^*$, is called M$^*$ and M$^*$ was limited to a single objective function.  We combine principles of dominance and subdimensional expansion to create a new algorithm, named multi-objective M$^*$ (MOM$^*$), which dynamically couples agents for planning only when those agents have to ``interact'' with each other. MOM$^*$ computes the complete Pareto-optimal set for multiple agents efficiently and naturally trades off  sub-optimal approximations of the Pareto-optimal set and computational efficiency. Our approach is able to find the complete Pareto-optimal set for problem instances with hundreds of solutions which the standard multi-objective A$^*$ algorithms could not find within a bounded time. %This paper presents a formal proof of the completeness and optimality of MOM*, followed by experimental results that demonstration the efficiency of MOM*. 

\end{abstract}

\graphicspath{{contents/paper_fig/}}

\section{Introduction}

Multi-agent path finding (MAPF), as its name suggests, determines an ensemble of paths for multiple agents between their respective start and goal locations. %vertices within a graph that describes the environment. 
Variants of MAPFs have received significant attention in the robotics community over the last decade~\cite{stern2019multi}. In a standard MAPF problem, each ensemble of paths is associated with a single objective such as the agent's total travel time or travel risk. However, in many real-world applications~\cite{wurman2008coordinating, hayat2017multi, soltani2002path}, each ensemble of paths is associated with multiple and sometimes conflicting objectives such as path length, travel risk, arrival time and other domain-specific measures. This article presents a natural generalization of the MAPF to include multiple objectives for multiple agents. We call this problem multi-objective multi-agent path finding (MOMAPF).%. , where agents have to optimize multiple objectives . 

\begin{figure}[htbp]
	\centering
% 	\vspace{-2mm}
	\includegraphics[width=0.96\linewidth]{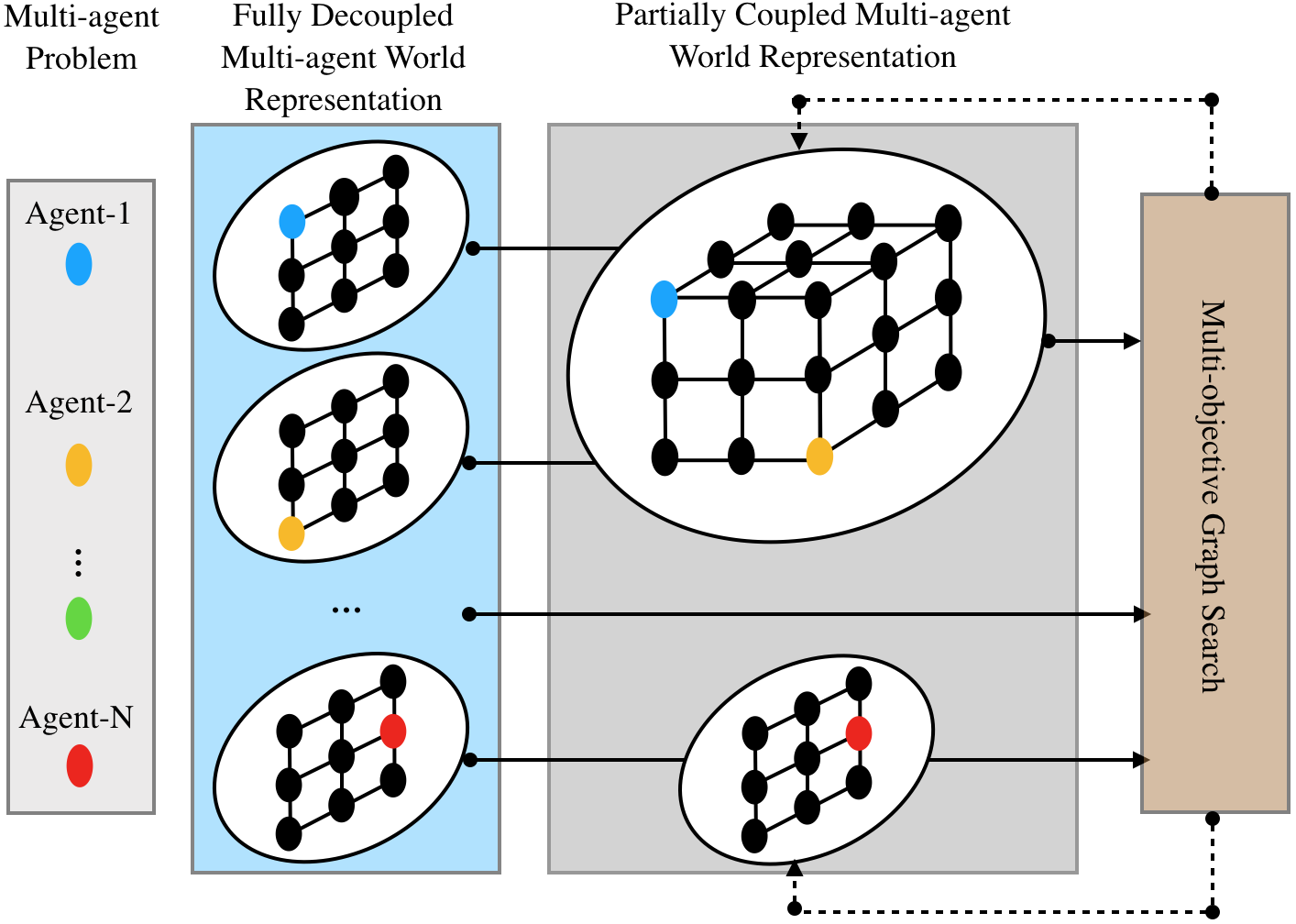}
 	\vspace{-4mm}
	\caption{A conceptual illustration of MOM$^*$. Agents initially plan their paths independently and are coupled together only when needed. The entire search process optimizes multiple objectives and yields all Pareto-optimal solutions.}
	\vspace{-6mm}
	\label{fig:MOM*}
\end{figure}

In the presence of multiple objectives, in general, there may not be a single optimal solution that optimizes all the objectives. 
Therefore, we seek a set of Pareto-optimal solutions for multi-objective problems~\cite{ehrgott2005multicriteria}.
A solution is Pareto-optimal if there exists no other solution that will yield an improvement in one objective without causing a deterioration in at least one of the other objectives. 
Finding a Pareto-optimal set for MOMAPF while ensuring conflict-free paths for agents in each solution is quite challenging~\cite{serafini1987some,yu2013structure_nphard} as the size of the Pareto-optimal set may grow exponentially with respect to the number of agents as well as the dimension of the search space.
Even though there are multi-objective single-agent search algorithms, such as NAMOA$^*$~\cite{namoa} and others~\cite{moastar,sedeno2019biobjective}, simply applying them to the joint configuration space of all agents may be inefficient. In this article, we present a new approach called multi-objective M$^*$ (MOM$^*$) that is complete and aims to find the set of all Pareto-optimal solutions for MOMAPF.
This work leverages our subdimensional expansion framework for single-objective MAPF~\cite{wagner2015subdimensional}, and existing multi-objective dominance techniques~\cite{ehrgott2005multicriteria}. 

%, we propose a new algorithm named Multi-objective M* (MOM*) that is able to compute the Pareto-optimal set for multiple agents efficiently.  To our limited knowledge, we are not aware of any MAPF algorithm that is able to find all Pareto-optimal solutions subject to multiple objectives. This work is one of our series of efforts~\cite{ren2021multi,ren2021subdimensional} to explore MAPF algorithms in a multi-objective setting.

Specifically, MOM$^*$ (Fig. \ref{fig:MOM*}) begins by letting agents follow their individual ``optimal'' policies subject to multiple objectives.
When an agent-agent conflict is found, MOM$^*$ couples those agents together in the same manner as M*~\cite{wagner2015subdimensional}, which locally increases the dimensionality of the search space and enables MOM$^*$ to plan a path in the joint space to resolve conflicts.
To compare two paths subject to multiple objectives, MOM$^*$ leverages the dominance rules \cite{ehrgott2005multicriteria}. 
As a result, MOM$^*$ plans in a relatively compact space to find all Pareto-optimal solutions and saves computational effort with respect to the existing graph search approaches such as NAMOA$^*$~\cite{namoa}, which plans in the joint space of agents. Additionally, with inflated heuristics~\cite{pearl1984intelligent}, MOM$^*$ can be modified to a bounded sub-optimal algorithm that trades off between computational efficiency and approximation quality.

%{\color{red}
To evaluate MOM$^*$, we ran tests in various grid-like maps from~\cite{stern2019multi}.
As a baseline, we apply NAMOA$^*$, a well-known graph search algorithm that computes the exact Pareto-optimal set to the joint graph of all agents to solve MOMAPF.
We also compared MOM$^*$ with another work from us, namely the multi-objective conflict-based search (MO-CBS) \cite{ren2021multi}, which extends conflict-based search to compute the exact Pareto-optimal set for MOMAPF.
We compared all three algorithms under multiple metrics, as functions of both number of agents (as conventional MAPF research does) as well as number of objectives.
Our results show that MOM$^*$ outperforms NAMOA$^*$ in terms of (1) success rates of finding all Pareto-optimal solutions, (2) number of state expanded and (3) run time over various test instances.
Comparing with MO-CBS, MOM$^*$ achieves better success rates and shorter run time in some maps. There is no leading algorithm over all instances for MOMAPF, which is similar to the observation in the single-objective MAPF~\cite{felner2017search}. In addition, we also corroborated the inflated MOM$^*$, an algorithm which finds sub-optimal solutions, with different heuristic inflation rates under various grids with up to twenty agents. Inflated MOM$^*$ demonstrates obviously better scalability than both MOM$^*$ and MO-CBS.
%}
%In addition, we also compared MOM* with our prior approach Multi-objective Conflict-based Search~\cite{ren2021multi} and shows the benefits of MOM*.

The rest of the article is organized as follows. Sec.~\ref{sec:prior} reviews the related work and Sec.~\ref{sec:problem} defines the MOMAPF problem. MOM$^*$ algorithm is described in Sec.~\ref{sec.method} and its properties are proved in Sec.~\ref{sec:analysis}. Numerical results are presented in Sec.~\ref{sec.result} with conclusions in Sec.~\ref{sec.conclud}.

\section{Related Work}\label{sec:prior}

\vspace{-1mm}
\subsection{Multi-agent Path Finding}
MAPF methods tend to fall on a spectrum from centralized to decentralized, trading off completeness and optimality for scalability. 
On the decentralized side of the spectrum, prioritized approaches~\cite{silver2005hca,ma2019searching} achieve scalability by decomposing the joint configuration space into a sequence of individual configuration spaces and plan for agents one after another in a pre-defined sequence, which is not guaranteed to be complete. 
In the middle of the spectrum, rule-based approaches~\cite{luna2011push,wang2008fast}, seek to overcome the curse of dimensionality by resolving agent-agent conflicts with pre-defined primitives. While the approaches compute scalable solutions, they lack completeness and optimality.
Unfortunately, these approaches are not complete and often form paths with undesired features. 
Reduction-based methods transform MAPF to known problems like network flow~\cite{yu2013structure_nphard}, SAT~\cite{surynek2017sat}, etc, and rely on the corresponding solvers.
Conflict-based search (CBS)~\cite{sharon2015conflict} is a two level search algorithm that finds optimal paths for agents and has been improved in many ways~\cite{boyarski2015icbs,li2020eecbs}. 
Finally, on the centralized side of the spectrum, A$^*$-based approaches~\cite{standley2010finding,goldenberg2014enhanced} leverages A$^*$-like search and plan in the joint configuration space of all the agents. 

Subdimensional expansion~\cite{wagner2015subdimensional}, as another method stands in the middle of the spectrum, bypasses the curse of dimensionality by dynamically modifying the dimension of the search space based on agent-agent conflicts.
This approach applies to many algorithms~\cite{wagner2012probabilistic, wagner2017path}, and inherits completeness and optimality, if the underlying algorithm already has these features. 
Combining subdimensional expansion with A$^*$ results in M$^*$~\cite{wagner2015subdimensional}. M$^*$ begins by computing a set of individual policies for agents and lets every agent follow its policy towards its goal, initially ignoring agent-agent conflicts. 
When a conflict is detected, the subset of agents in conflict are locally coupled together to form a new search space with increased dimensions, where a new search is conducted to resolve the conflict. 
In this work, we combines subdimensional expansion with multi-objective search.

\vspace{-1mm}
\subsection{Multi-objective Search}
Multi-objective optimization (MOO) has been applied to single-agent path planning~\cite{moastar,namoa,ulloa2020simple}, motion planning~\cite{tesch2013expensive}, reinforcement learning~\cite{roijers2013survey}, etc. 
%and other problems~\cite{deb2001multi,montoya2013multiobjective}. 
With respect to MOO for multiple agents, the only known work (we are aware of) is on evolutionary algorithms~\cite{weise2020momapf} for a related variant of MOMAPF. Specifically, in \cite{weise2020momapf}, the authors consider a variant of MOMAPF where conflicts between the agent's paths is modeled in one of the objectives and not as a constraint; as a result, a solution in the Pareto-optimal set may have agent's paths that conflict with each other. In addition, the work in \cite{weise2020momapf} assumes no wait times for agents at any vertex in the graph which greatly reduces the size of the Pareto-optimal set. In this work, we allow for agents to wait at any vertex in the graph and seek conflict-free paths for agents in any feasible solution as in the standard MAPF.

%The only known work related to MOMAPF developed evolutionary algorithms for a related variant of MFor applying MOO to MAPF, the earliest (and possibly the only) known work is done by Weise et al.~\cite{weise2020momapf}, where multi-objective evolutionary algorithms (MOEA) are applied to solve the problem.
%We learnt from this work that simply applying MOEA can neither guarantee finding the whole set of Pareto-optimal solutions nor provide bounded sub-optimality of the solutions. The proposed MOM* in this work provides all these features. 
%}

Our approach is closely related to conventional multi-objective (graph) search problems~\cite{loui1983optimal}, which requires finding all Pareto-optimal paths connecting start and goal vertices within a given graph subject to multi-objectives. 
One of the basic challenges in multi-objective search is to compute the exact Pareto-optimal set, the size of which can grow exponentially with respect to the number of nodes in a graph \cite{serafini1987some}. 
To overcome the challenge, algorithms like A$^*$~\cite{astar}, as a best-first search method, are extended to multi-objective A$^*$ (MOA$^*$) \cite{moastar} and later improved by A New Approach to MOA$^*$ (NAMOA$^*$)~\cite{namoa}. 
Even though MOA$^*$-type algorithms can be applied to the product of the configuration spaces of the agents, numerical results in Sec. \ref{sec.result} show that it is significantly less efficient compared to the proposed MOM$^*$, which is an algorithm that benefits from the techniques in both MAPF communities and multi-objective search communities.

%{\color{red}
In parallel with this work, we are also working on multi-objective conflict-based search (MO-CBS)~\cite{ren2021multi}, which extends the popular conflict-based search algorithm to compute the exact Pareto-optimal set for MOMAPF.
Our results show that there is no leading algorithm in all maps.
Additionally, the approximation algorithm proposed in this work, inflated MOM$^*$, achieves better scalability than both MOM$^*$ and MO-CBS while providing approximated Pareto-optimal set with bound guarantees.
%}

%{\color{red}
%In our series of efforts to explore MAPF in multi-objective settings, we leverage both CBS and M* and propose MO-CBS~\cite{ren2021multi} and MOM*~\cite{ren2021subdimensional} (this work). From our experimental results, with two agents, as the number of objectives varies from one to four, MOM* outperforms MO-CBS in terms of both success rates and average run time. Additionly, MOM* can be readily extended to a bounded sub-optimal version by inflating heuristics which scales much better. }

\section{Problem Definition}\label{sec:problem}

Let index set $I=\{1,2,\dots,N\}$ denote a set of $N$ agents. 
All agents move in a workspace represented as a graph $G=(V,E)$, where the vertex set $V$ represents all possible locations of agents and the edge set $E = V \times V$ denotes the set of all possible actions that move an agent between any two vertices in $V$.
An edge between two vertices $u, v \in V$ is denoted as $(u,v)\in E$ and the cost of an edge $e \in E$ is an $M$-dimensional strictly positive real vector: $\text{cost}(e) \in (\mathbb{R}^{+})^M\backslash\{0\}$,
with $M$ being a positive integer.\footnote{Wait-in-place action of an agent can be treated as a self-loop in a graph, which is an edge that connects a vertex to itself. In this work, the cost vector of a self-loop is also required to be strictly positive.} 

In this work, we use a superscript $i \in I$ over a variable to represent the specific agent to which the variable belongs. For example, $v^i\in V$ means a vertex with respect to agent $i$. 
Let $\pi^i(v^i_{1}, v^i_{\ell})$ be a path that connects vertices $v^i_{1}$ and $v^i_{\ell}$ via a sequence of vertices $(v^i_{1},v^i_{{2}},\dots,v^i_{\ell})$ in $G$. 
Let $g^i(\pi^i(v^i_{1}, v^i_\ell))$ denote the $M$-dimensional cost vector associated with the path, which is the sum of the cost vectors of all the edges present in the path, i.e. $\displaystyle g^i(\pi^i(v^i_{1}, v^i_{\ell})) = \Sigma_{j=1,2,\dots,{\ell-1}} \text{cost}(v^i_{{j}}, v^i_{{j+1}})$. 
 
All agents share a global clock. Each action, either wait or move, for any agent requires one unit of time. 
All the agents start their paths at time $t=0$.
Any two agents $i,j \in I$ are said to be in conflict if one of the following two cases happens. The first case is a ``vertex conflict'' where two agents occupy the same vertex at the same time. The second case is an ``edge conflict'' {(also called swap conflict)} where two agents
move through the same edge from opposite directions between times $t$ and $t+1$ for some $t$.

Let $v_o^i, v^i_f \in V$ respectively denote the initial location and the final destination of agent $i$. Without loss of generality, to simplify the notations, we also refer to a path $\pi^i(v^i_{o}, v^i_{f})$ for agent $i$ between its initial and final locations as simply $\pi^i$. Let $\pi=(\pi^1,\pi^2,\dots, \pi^N)$ represent a joint path for all the agents. The cost vector of this joint path is defined as the vector sum of the individual path costs over all the agents, $i.e.$, $g(\pi) = \Sigma_i g^i(\pi^i)$. 
To compare any two joint paths, we compare the cost vectors corresponding to them. Given two vectors $a$ and $b$, $a$ \textit{dominates} $b$ if every component in $a$ is no larger than the corresponding component in $b$ and there exists at least one component in $a$ that is strictly less than the corresponding component in $b$.
Formally, it is defined as follows:
% \begin{definition}[Dominance \cite{namoa}]
% 	Given two vectors $a$ and $b$ of length $M$, $a$ dominates $b$, symbolically $a \succeq b$, if $\forall m \in \{1,2,\dots,M\}$, $a(m) \leq b(m)$, and $a \neq b$.
% \end{definition}
\begin{definition}[Dominance \cite{namoa}]
	Given two $M$-dimensional vectors $a$ and $b$, $a$ dominates $b$, symbolically $a \succeq b$, if $\forall m \in \{1,2,\dots,M\}$, $a(m) \leq b(m)$, and there exists $m \in \{1,2,\dots,M\}$ such that $a(m) < b(m)$.
\end{definition}
If $a$ does not dominate $b$, we represent this non-dominance as $a \nsucceq b$. Any two solution joint paths are non-dominated if the corresponding cost vectors do not dominate each other. The set of all non-dominated solutions is called the {\it Pareto-optimal} set.
%All the agents start their paths at time $t=0$.
MOMAPF problem aims to find an maximal subset of the Pareto-optimal set, where any two solutions in this subset do not have the same cost vector.
%at finding all the conflict-free, non-dominated solutions where the path for every agent $i\in I$ in each solution starts at $v^i_o$ and reaches $v^i_f$ at some later time.

\section{Multi-Objective M*}\label{sec.method}

%\subsection{Problem Definition}\label{sec.problem_def}
%\input{../contents/problem}

%\section{MOM*}\label{sec.momstar}

\vspace{-1mm}
\subsection{Notation}\label{sec.math_def}

Let $\mathcal{G}=(\mathcal{V},\mathcal{E}) = \underbrace{G \times G \times \dots \times G}_{\text{$N$ times}}$ denote the joint graph which is the Cartesian product of $N$ copies of graph $G$, where each vertex $v \in \mathcal{V}$ represents a joint vertex and $e \in \mathcal{E}$ represents a joint edge that connects a pair of joint vertices.
The joint vertex corresponding to the initial locations of the agents is $v_o = (v^1_o,v^2_o,\cdots,v^N_o)$. In addition, let $\pi(u,v), u,v \in \mathcal{V}$ represent a joint path, which is a tuple of $N$ individual paths, $i.e.$,  $\pi(u,v) = (\pi^1(u^1,v^1), \cdots, \pi^N(u^N,v^N))$. 

The state of the agents is denoted as $s=(v,g)$, which represents a partial solution: $v$ represents the joint vertex occupied by the agents and $g$ denotes the cost vector of a joint path connecting $v_o$ and $v$. 
For the rest of the work, let $v(s)$ denote the joint vertex in state $s$ and $g(s)$ denote the cost vector. Similarly, let $v^i(s)$ denote the individual vertex in $v(s)$ corresponding to agent $i \in I$. Besides, we say a state $s_k$ \emph{visits} joint vertex $v_l$ if $v(s_k)=v_l$.

Similar to the A$^*$ algorithm, a heuristic (cost) vector is associated with each joint vertex. 
%{\color{red}
The heuristic vector of $v \in \mathcal{V}$ is denoted as $h(v)$, which is the sum of the (individual) heuristic vectors of every individual vertex $v^i \in v$, $i.e.$ $h(v)=\Sigma_{i\in I}h(v^i)$ and $h(v^i)$ is an component-wise underestimate of the cost vectors of all Pareto-optimal paths from $v^i$ to $v^i_f$. Also, the $f$-vector associated with a state $s$ is defined as $f(s):=g(s)+w \cdot h(v(s))$, where $w$ is the heuristic inflation rate with value no less than one.
%}

Let ``collision set'' $I_C(s) \subseteq I$ represent a set of agents in conflict at state $s$.
To detect conflicts, collision function $\Psi:\mathcal{V} \times \mathcal{V} \rightarrow 2^{I}$ is introduced to check if there is any vertex or edge conflicts given two adjacent joint vertices.
Collision function $\Psi$ returns either an empty set if no conflict is detected, or a set of agents that are in conflict.

As in the well-known A$^*$ algorithm \cite{astar}, at any stage of the algorithm, let OPEN be the open list which contains candidate states to be expanded and let CLOSED represent the set of states that have already been expanded. In addition, states in the OPEN and CLOSED list are organized based on their joint vertices: given a joint vertex $v$, let OPEN($v$) and CLOSED($v$) denote the subset of open and closed states that share the same $v$. 

\vspace{-1mm}
\subsection{Algorithm Overview}\label{sec.algo_overview}
%{\color{red}
The proposed MOM$^*$ algorithm avoids directly search over the joint graph. Intuitively, MOM$^*$ gains computation efficiency by (1) only considering the joint graph of a subset of agents when the agents have to interact to avoid conflicts, and (2) frequently pruning the partial solutions that are dominated.
%}

The pseudo code\footnote{For readers that are familiar with M$^*$~\cite{wagner2015subdimensional}, we highlight the differences between MOM$^*$ and M$^*$ in blue in the pseudo-code.} of MOM$^*$ is shown in Algorithm \ref{alg:mom*}. For initialization, the initial state $s_o=(v_o,0)$, where $0$ denotes a zero vector, is added to OPEN.
In each iteration of the search loop (from line 4), a state $s_k$ with a non-dominated $f$-vector is popped from OPEN.
Next, the joint vertex $v(s_k)$ is compared with the final joint vertex $v_f$.
\begin{itemize}
    \item If $v(s_k)=v_f$, a non-dominated solution can be obtained (line 6) by iteratively back-tracking the parent of each state from $s_k$ to $s_o$. In addition, the $f$-vector of this non-dominated solution $f(s_k)$ is used to ``filter'' all the candidate states in OPEN by calling a sub-procedure (Sec.~\ref{sec:terminate_cond}), which prunes any candidate states with $f$-vectors dominated by $f(s_k)$ (note that $h(v_k)=h(v_f)=0$ and $f(s_k)=g(s_k)$).
    
    \item If $v(s_k) \neq v_f$, state $s_k$ (from line 10) is expanded by considering only a limited set of neighbor states $S^{ngh}$ of $s_k$ (Sec.~\ref{sec:limited_ngh}), which avoids searching over the joint graph directly and reduces the branching factor during the search. The $S^{ngh}$ of $s_k$ is controlled by the collision set $I_C(s_k)$.
    For each neighbor $s_l \in S^{ngh}$, collision function $\Psi(v(s_k),v(s_l))$ is called and the resulting collision set is back-propagated (Sec.~\ref{sec:limited_ngh}). 
    This back-propagation updates the collision set of the ancestor states if needed and inserts those ancestor states into OPEN for re-expansion with a larger limited neighbor set.
    Then, state $s_l$ is compared with any other state $s'$ that visits $v(s_l)$ ($i.e.$ $v(s')=v(s_l)$). If cost vector $g(s_l)$ is dominated by or equal to $g(s')$, $s_l$ is pruned and a sub-procedure is called to handle collision set back-propagation when dominance happens (Sec.~\ref{sec:compare}). Otherwise, $s_l$ is inserted into OPEN as a candidate state for future expansion.
\end{itemize}
When OPEN becomes {\it empty}, there is no candidate state with a non-dominated $f$-vector. The algorithm then terminates and returns $\mathcal{S}$, which is a set that contains all cost-unique non-dominated solutions for MOMAPF.

\begin{algorithm}[htbp]
	\caption{Pseudocode for MOM$^*$}\label{alg:mom*}
	\small
	\begin{algorithmic}[1]
		\State{initialize OPEN and {\color{blue}OPEN($v_o$)} with $s_o=(v_o,{0})$}
		\State{{\color{blue}$\mathcal{S}$ $\gets \emptyset$} \Comment{A set of solutions}}
		\While{OPEN not empty} \Comment{Main search loop}
		
		\State{{\color{blue}$s_k \gets$ OPEN.pop() } }
		\State{{\color{blue}move $s_k$ from OPEN($v(s_k)$) to CLOSED($v(s_k)$)}}
		\State{\textbf{if} $v(s_k)=v_f$ \textbf{then}}
		\State{\indent $\pi \gets$ \text{Reconstruct($s_k$)}} \Comment{Reconstruct joint path}
		\State{{\color{blue}\indent add $\pi$ to $\mathcal{S}$}}
		\State{\indent {\color{blue}\textbf{FilterOpen}($s_k$)}} \Comment{Use $s_k$ to filter open list }
		
		\State{$S^{ngh} \gets$ \textbf{GetNeighbor}($s_k$) } 
		\ForAll{$s_l \in S^{ngh}$ }
		\State{$I_C(s_l) \gets I_C(s_l) \bigcup \Psi(v(s_k),v(s_l))$}
		\State{\textbf{BackProp($s_k$, $I_C(s_l)$)}}
		\State{$f(s_l) \gets g(s_l)$ + $w \cdot h(v_l)$}
		\State{\textbf{if} $\Psi(v(s_k),v(s_l)) \neq \emptyset$}
		\State{\indent \textbf{continue}}
		\State{\textbf{if} {\color{blue} \textbf{Compare}($s_l$)} \textbf{then} }
		\State{\indent add $s_l$ to OPEN and {\color{blue}OPEN($v(s_l)$)} }
		\State{\indent add $s_k$ to back\_set($s_l$)}
		\State{\indent parent($s_l$) $\gets s_k$}
		\State{\textbf{else}} \Comment{Back-propagation due to dominance}
		\State{\indent{\color{blue}\textbf{DomBackProp($s_k,s_l$)}}}
		\EndFor
		\EndWhile \label{}
		\State{\textbf{return} $\mathcal{S}$}
	\end{algorithmic}
\end{algorithm}

\vspace{-1mm}
\subsection{Relationship to M$^*$}\label{sec:relation_to_mstar}
The MOMAPF problem, as defined in Sec.~\ref{sec:problem} generalizes the conventional (single-objective) MAPF: when $M=1$, the MOMAPF problem becomes a conventional (single-objective) MAPF problem and the solution set consists of only one solution with the optimal cost value. Similarly, the proposed MOM$^*$ algorithm can also be regarded as a generalization of M$^*$ to handle multiple objectives. Specifically, when $M=1$, MOM$^*$ solves (single-objective) MAPF in the following sense.
\begin{itemize}
    \item State comparison becomes ``$\leq$'' (no larger than) relationship between two scalar cost values.
    \item In every search iteration, a state with a non-dominated ($i.e.$ minimum) cost value in OPEN is popped, which guarantees that the first solution identified is the one with the minimum cost value (denoted as $C^*$).
    \item All candidate states in OPEN have cost no less than $C^*$ and are thus pruned in the sub-procedure that filters OPEN with $C^*$, which leads to the termination of the algorithm as OPEN becomes empty.
    \item Procedure DomBackProp and maintaining a set OPEN($v$)$\bigcup$CLOSED($v$) at each joint vertex $v \in \mathcal{V}$ are redundant when $M=1$ since there is only \emph{one} optimal path from $v_o$ to any $v \in \mathcal{V}$ (up to tie).
\end{itemize}

\vspace{-1mm}
\subsection{Key Parts of Our Approach}
\subsubsection{State Space}\label{sec:state_space}
MOM$^*$ defines its search state to be a tuple $(v,g)$ consisting of a joint vertex $v$ and a cost vector $g$, which identifies a partial solution from $v_o$ to $v$ with cost vector $g$. The reason for such a state definition is rooted at the key difference between single-objective and multi-objective search problems: while there is one optimal path from $v_o$ to any $v$ in single-objective settings, there are \emph{multiple} non-dominated paths $\pi(v_o,v)$ for any $v \in \mathcal{V}$ for multi-objectives. This state definition allows the algorithm to differentiate between paths $\pi(v_o,v)$ with different cost vectors.

\subsubsection{Pareto Policy}\label{sec:pareto_policy}
% \paragraph{Pareto Policy}\label{sec:pareto_policy}
In multi-objective settings, for every agent $i\in I$, there exists multiple non-dominated paths from $v^i_o$ to a vertex $v^i \in V$ and there are also multiple non-dominated paths $\pi^i(v^i,v^i_f)$. Therefore, the optimal policy $\phi^i(v^i)$ maps $v^i$ to \emph{multiple} neighbor vertices, each of which is along some non-dominated path $\pi^i(v^i,v^i_f)$. 
This differs from the concept of the optimal policy used in M$^*$, where $\phi^i(v^i)$ maps $v^i$ to only \emph{one} neighbor along an optimal path from $v^i$ to $v^i_f$.
To differentiate, we use term \emph{Pareto policy} to describe an individual policy $\phi^i$ in multi-objective settings for the rest of the article.
To compute $\phi^i,\forall i \in I$, MOM$^*$ runs single-agent multi-objective search~\cite{namoa} backwards from $v^i_f$ to all vertices in $G$. 

\subsubsection{Limited Neighbors}\label{sec:limited_ngh}
% \paragraph{Limited Neighbors}\label{sec:limited_ngh}
Limited neighbors is a concept originally introduced in M$^*$ and generalized here (by combining Pareto policies) for MOM$^*$. The limited neighbors $S^{ngh}_k$ of $s_k$ is a set of neighbor states that can be reached from $s_k$. For each agent $i$, if $i \notin I_C(s_k)$, agent $i$ is only allowed to follow its Pareto policy $\phi^i(v^i_k)$. If $i \in I_C(s_k)$, agent $i$ is allowed to visit any neighbor of $v^i_k$ in $G$. Formally, 
\begin{gather}\label{equ.limited_ngh}
S^{ngh}_k = \{ s_l \bigg|
\begin{cases}
v^i(s_l) \in \phi^i(v^i(s_k)) &\mbox{if } i \notin I_C(s_k) \\
v^i(s_l) | (v^i(s_k),v^i(s_l)) \in E & \mbox{if } i \in I_C(s_k)
\end{cases},\nonumber\\
g(s_l) = g(s_k) + \Sigma_{i\in I} (\text{cost}(v^i(s_k),v^i(s_l)))
\; \}.
\end{gather}

Limited neighbors of a state $s$ varies once $I_C(s_k)$ changes, which dynamically modifies the sub-graph embedded in joint graph $\mathcal{G}$ that can be reached from $s_k$. 
Collision set $I_C(s_k)$ is updated (enlarged) recursively when Algorithm \ref{alg.backprop} is called (line 12 in Algorithm \ref{alg:mom*}). To keep track of the states that should be back-propagated to, a data structure ``back\_set'' is defined at every search state. Intuitively, the back\_set at state $s_k$ contains all parent states from which $s_k$ is ever reached. When $I_C(s_k)$ is enlarged, $I_C(s_k)$ is back-propagated to every state in back\_set($s_k$).

\begin{algorithm}[htbp]
	\caption{Pseudocode for BackProp}\label{alg.backprop}
	\small
	\begin{algorithmic}[1]
		\State{INPUT: $s_k$, $I_C(s_l)$}
		\State{\textbf{if} $I_C(s_l) \nsubseteq I_C(s_k)$ \textbf{then}}
		\State{\indent $I_C(s_k) \gets I_C(s_l) \bigcup I_C(s_k)$}
		\State{\indent\textbf{if} $s_k \notin$ OPEN \textbf{then}}
		\State{\indent\indent add $s_k$ to OPEN and {\color{blue}OPEN($v(s_k)$)} }
		\State{\indent\indent {\color{blue}remove $s_k$ from CLOSED($v(s_k)$)} }
		\State{\indent\textbf{for all} $s_k' \in$ back\_set($s_k$) \textbf{do}}
		\State{\indent\indent\textbf{BackProp}($s_k'$, $I_C(s_k)$)}
	\end{algorithmic}
\end{algorithm}

When back-propagating a collision set $I_C(s_l)$ to a state $s_k \in$ back\_set($s_l$), if $I_C(s_k)$ is not a super set of $I_C(s_l)$, then $I_C(s_k)$ is updated by taking the union of $I_C(s_l)$ (line 3). 
In addition, $s_k$ is re-inserted into OPEN (and OPEN($v(s_k)$)) from CLOSE (and CLOSE($v(s_k)$)), which makes MOM$^*$ expand $s_k$ again with a possibly larger limited neighbor set. 

\subsubsection{State Comparison}\label{sec:compare}
Given a joint vertex $v\in \mathcal{V}$, let $\alpha(v), v\in \mathcal{V}$ denote the ``frontier set at $v$'': a subset of states that visits $v$ with non-dominated cost vectors. In MOM$^*$, set OPEN($v$)$\bigcup$CLOSE($v$) forms a super set of $\alpha(v)$. When MOM$^*$ expands $s_k$ and generates state $s_l$, to decide whether $s_l$ should be pruned or not (line 17 in Algorithm \ref{alg:mom*}), $g(s_l)$ is compared with the cost vector $g(s'), \forall s' \in$ OPEN($v(s_l)$)$\bigcup$CLOSE($v(s_l)$). 
\begin{itemize}
    \item If $g(s_l)$ is dominated by (or equal to) $g(s')$, $s_l$ can not be part of $\alpha(v)$ and is thus pruned. In addition, MOM$^*$ back-propagates the collision set at $s'$ to $s_k$ and then add $s_k$ to the back\_set($s'$) before discarding $s_l$. By doing so, MOM$^*$ keeps updating collision sets of any ancestor states of $s_l$ after $s_l$ is pruned. %which eventually ensures MOM* finds all the Pareto-optimal paths with dominance pruning.
    \item Otherwise, $s_l$ may represent a non-dominated partial solution from $v_o$ to $v$ and is therefore inserted into OPEN and OPEN($v(s_l)$) and becomes part of set OPEN($v(s_l)$)$\bigcup$CLOSE($v(s_l)$) permanently.
\end{itemize}

\begin{algorithm}
	\caption{Pseudocode for DomBackProp}\label{alg.dom_back_prop}
	\small
	\begin{algorithmic}[1]
		\State{INPUT: $s_k, s_l$} \Comment{$s_l$ is generated from $s_k$}
		{\color{blue}
		\State{\textbf{for all} $s'$ in OPEN($v(s_l)$)$\bigcup$CLOSED($v(s_l)$) \textbf{do}}
		\State{\indent\textbf{if} $g(s') \succeq g(s_l) \text{ or } g(s') = g(s_l)$ \textbf{then}}
		\State{\indent\indent\textbf{BackProp}($s_k$, $I_C(s')$)}
		\State{\indent\indent add $s_k$ to back\_set($s'$)}
		}
	\end{algorithmic}
\end{algorithm}

\subsubsection{Termination Condition and Heuristic Inflation}\label{sec:terminate_cond}
%\paragraph{Termination Condition}
Different from M$^*$, which terminates when the first solution is identified, MOM$^*$ terminates only when OPEN is empty, to identify all Pareto-optimal solutions. 
For a MOMAPF problem, however, OPEN is often prohibitively huge since the joint vertex space grows exponentially with respect to the number of agents.
To avoid unnecessary state expansion, multi-objective search algorithms often use the cost vector of a state $s_k$ that visits $v_f$ to filter candidates states in OPEN and prune those with dominated cost vectors.
%To save unnecessary expansion of states from OPEN, multi-objective search algorithms typically use Algorithm \ref{alg:filteropen} uses the cost vector of a state $s_k=(v_k,g_k)$ that reaches $v_f$ to filter candidates states in OPEN and prune ones with cost vector dominated by $g_k$.
%{\color{red}
In this work, when $v(s_k)=v_f$, we compare the $f$-vectors of both $s_k$ and the candidate states in OPEN. States in OPEN with dominated $f$-vectors are filtered.
%}
As the heuristic vector is an component-wise underestimate of the cost-to-goal, when $w=1$, MOM$^*$ algorithm does not prune any candidate states that are part of a non-dominated solution. 

The use of $f$-vectors in the filtering procedure allows MOM$^*$ to utilize inflated heuristics to trade off between search efficiency and bounded sub-optimality, as commonly done in A$^*$~\cite{pearl1984intelligent} or M$^*$-based algorithms~\cite{wagner2015subdimensional}. 
For MOM$^*$, when $w>1$, it is guaranteed that for any Pareto-optimal solution with cost vector $g^*$, inflated MOM$^*$ is able to find a sub-optimal solution with cost vector $g$ that's in the ``proximity'' of $g^*$ and such proximity is bounded by $w$. Detailed proof is provided in Sec.~\ref{sec:analysis}.

\begin{algorithm}[htbp]
	\caption{Pseudocode for FilterOpen}\label{alg:filteropen}
	\small
	\begin{algorithmic}[1]
		\State{INPUT: $s_k$} \Comment{$v(s_k)=v_f$}
		{\color{blue}
		\State{\textbf{for all} $s_l$ in OPEN \textbf{do}}
		\State{\indent\textbf{if}$f(s_k) \succeq f(s_l) \text{ or } f(s_k) = f(s_l)$ \textbf{then}}
		\State{\indent\indent move $s_l$ from OPEN($v(s_l)$) to CLOSED($v(s_l)$)}
		\State{\indent\indent remove $s_l$ from OPEN} }
	\end{algorithmic}
\end{algorithm}

\section{Theoretical Properties of MOM*}\label{sec:analysis} 

%\subsection{Preliminary}
%\paragraph{Preliminary}
%In this section, we discuss the properties of MOM*. We put some of the proof in the appendix and give the high-level idea of the proof.

%Section \ref{sec:optimal} shows that in finite time, MOM* either shows there is no joint path between $v_o$ and $v_f$, or returns the entire Pareto-optimal set. Section \ref{sec:sub_optimal} shows that with heuristic inflation rate $w>1$, for any Pareto-optimal path $\pi_*$, MOM* guarantees to return a solution $\pi$ with cost vector $g(\pi)$ that is element-wise bounded by $wg(\pi_*)$. 

%\subsection{Completeness and Optimality}\label{sec:optimal}
%For two cost vectors $g_k$ and $g_l$, let $g_k \nsucceq g_l$ represent that $g_k$ does not dominate $g_l$.

\begin{theorem}[Completeness and optimality]\label{thm:find_all_pareto}
	If there is no solution, MOM$^*$ terminates in finite time and reports failure; otherwise, MOM$^*$ finds all cost-unique Pareto-optimal joint paths connecting $v_o$ and $v_f$.
\end{theorem}

\begin{proof}
We provide a sketch proof to concisely highlight the overall flow and key steps.
Let $\Pi_*$ denote a set of all cost-unique conflict-free Pareto-optimal joint paths connecting $v_o$ and $v_f$ in $\mathcal{G}$.
Let $\mathcal{G}^{sch}$ represent the \emph{search graph}: the sub-graph of $\mathcal{G}$ that is reachable from $v_o$ by iteratively following limited neighbors. By the construction of MOM$^*$ algorithm (Alg.~\ref{alg:mom*}), during the search process, MOM$^*$ either grows $\mathcal{G}^{sch}$ by modifying collision set at search states or conducts A$^*$-like search in $\mathcal{G}^{sch}$ by selecting a candidate state with non-dominated cost vector for expansion from OPEN. 

We first claim that (claim 1) MOM$^*$ can identify all cost-unique Pareto-optimal solutions within $\mathcal{G}^{sch}$ because MOM$^*$ expands a state by generating all possible neighbors of that state in $\mathcal{G}^{sch}$ and terminates only when OPEN is empty. Thus, if $\mathcal{G}^{sch}$ contains all $\pi_* \in \Pi_*$, MOM$^*$ can find all of them. Next, we claim that If there exists a $\pi_* \in \Pi_*$ that is not contained in $\mathcal{G}^{sch}$, then there must exist a corresponding joint path $\pi'$ such that (1) there is a conflict along $\pi'$ and (2) $g(\pi_*) \nsucceq g(\pi')$. Here, condition (2) guarantees that states along $\pi'$ will by expanded by MOM$^*$ before termination and condition (1) guarantees that a conflict will be identified during the state expansion, which grows $\mathcal{G}^{sch}$ and eventually includes $\pi^*$ into $\mathcal{G}^{sch}$.
Therefore, combining with the previous claim (claim 1), MOM$^*$ computes all cost-unique Pareto-optimal solutions in $\Pi_*$.

If there is no feasible joint path from $v_o$ to $v_f$ in $\mathcal{G}$, MOM$^*$ modifies $\mathcal{G}^{sch}$ for a finite number of times because $\mathcal{G}$ is finite and there is only a finite number of partial solutions from $v_o$ to any other joint vertices in $\mathcal{G}$.\footnote{There is no need to bound the time horizon for search since paths with unnecessary waits are pruned by dominance when compared with the corresponding paths without unnecessary waits.} This guarantees that MOM$^*$ terminates in finite time without returning a solution.
\end{proof}

%\subsection{Bounded Sub-optimality}\label{sec:sub_optimal}

\begin{theorem}[Bounded sub-optimality]\label{thm:suboptm}
	When a heuristic is inflated by a factor of $w>1$, for any Pareto-optimal path $\pi^*$ with non-dominated cost vector $g^*$, within $\mathcal{S}$ returned by MOM$^*$, there exists a joint path $\pi$ with cost vector $g$ such that $g(m) < w \cdot g^*(m), m=1,2,\dots,M$.
\end{theorem}

\begin{proof}
	Let $s$ be a state expanded by MOM$^*$ with $v(s)=v_f$ and let $s'$ be a state that is filtered from OPEN when compared with $s$ in Algorithm \ref{alg:filteropen}.
	Let $h^*(v(s'))$ be {the true cost vector of an arbitrary Pareto-optimal path} from $v(s')$ to $v_f$ and let $f^*(s')=g(s')+h^*(v(s'))$.
	
	As $h(v(s'))$ component-wise underestimates $h^*(v(s'))$, $w \cdot f^*(s') \geq w(g(s')+h(v(s')))$.
	Since $w>1$, $w(g(s')+h(v(s'))) > f(s')=g(s')+wh(v(s'))$.
	Since $s'$ is filtered from OPEN when compared with $s$, $f(s')=g(s')+wh(v(s')) > f(s)=g(s)+wh(v(s))$.
	As $v(s)=v_f$, it means $h(v(s))=0$ and $f(s)=g(s)+wh(v(s))=g(s)$.
	Put them together, we have $wf^*(s') > g(s)$.
	If the filtered state $s'$ is part of a Pareto-optimal path (with Pareto-optimal cost vector $f^*(s')$), then MOM$^*$ finds a solution with cost vector $ g(s) < wf^*(s')$.
\end{proof}

% \subsection{Complexity}\label{sec:complexity}
% \textbf{Richard:
% Not sure if we want to elaborate on complexity. One of the reviewer points out that our problem might be #P hard. I did some search and it turns out that I need to read deep into some book on this in order to be really rigorous. I prefer we do not stress on complexity in this work.
% }

%\section{Discussion}\label{sec.discuss}
%\input{contents/discuss}

\section{Numerical Results}\label{sec.result}

%\graphicspath{{contents/paper_fig/}}

We implemented both MOM$^*$ and NAMOA$^*$~\cite{namoa} in Python. 
The NAMOA$^*$ is applied to the the joint graph $\mathcal{G}$ of agents and serves as a baseline approach.
%{\color{red}
As we are working on extending conflict-based search \cite{sharon2015conflict} to multi-objective conflict-based search (MO-CBS) \cite{ren2021multi}, the preliminary results of a Python implementation of MO-CBS is also reported here for comparison.
%}
All algorithms are tested on a computer with a CPU of Intel Core i7 and 16GB RAM. 
To test the algorithms, we selected four maps (grids) from different categories~\cite{stern2019multi} and generate an un-directed graph $G$ by making each grid four-connected.
To assign cost vectors to edges, we first assign every agent a cost vector $a^i$ of length $M$ (the number of objectives) for all agents $i$, and assign every edge $e$ in graph $G$ a scaling vector $b(e)$ of length $M$, where each component in both $a^i$ and $b(e)$ are randomly sampled from integers in $[1,10]$. The cost vector for agent $i$ to go through an edge $e$ is the component-wise product of $a^i$ and $b(e)$. 
If agent $i$ waits in place, the cost vector incurred is $a^i$. 
We tested the algorithms with different heuristic inflation rates $w$, different number of objectives $M$ and different number of agents $N$.
We limited the computation time of each instance to {\it five} minutes.

\vspace{-0mm}
\subsection{Implementation of heuristics and Pareto policies}
In our implementation, to compute heuristics $h(v), v\in \mathcal{V}$, we first run \emph{single-agent} NAMOA$^*$ search backwards from $v^i_f$ to all other vertices in $G$ for all agents $i\in I$, which computes the set of all Pareto-optimal cost vectors $\{g^i_*\}$ of paths $\pi^i(v^i_f,u)$ that connects $v^i_f$ with $u \in V$ for all $i \in I$. 
Then, for each (individual) vertex $u$ in $G$, the component-wise minimum over all vectors in $\{g^i_*\}$ forms a cost vector $h^i(u)$ which underestimates the cost vectors of any path from $u$ to $v^i_f$.
Finally, the heuristic value of a joint vertex $v\in \mathcal{V}$ is computed by $h(v)=\Sigma_{i\in I} h^i(v^i)$, where $v^i$ is the individual vertex contained in $v$ for agent $i$.
%{\color{red}
As this computation finds the individual Pareto-optimal paths at each node for each agent, the Pareto policy for each agent is also constructed.
The heuristics are computed before the search begins for both NAMOA$^*$ and MOM$^*$, and the computation time is included in the run time we report next.
%}
%This heuristic is implemented as a pre-processing step in both MOM* and NAMOA* (for multi-agents), {\color{red} }

\vspace{-0mm}
\subsection{Experiments with Different Numbers of Objectives}

\begin{figure}[htbp]
	\centering
	\vspace{-1mm}
	\includegraphics[width=1.00\linewidth]{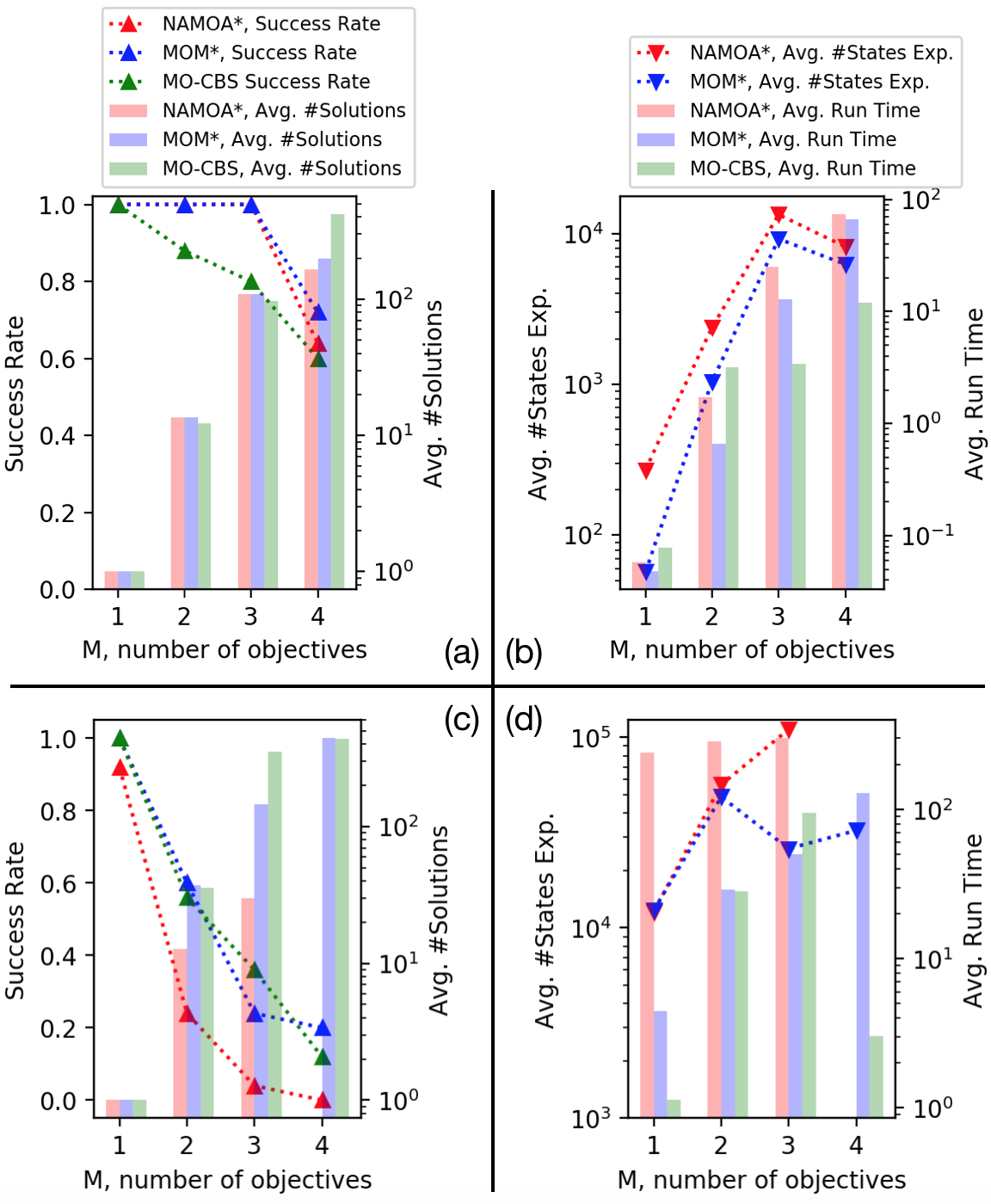}
	\vspace{-6mm}
	\caption{Comparing MOM$^*$, NAMOA$^*$ and MO-CBS with a fixed number of agents $N$ and varying number of objectives $M$. In plots (a) and (b), $N=2$. In plots (c) and (d), $N=4$.}
	% 		\vspace{-4mm}
	\label{fig:measures_vs_m_nr}
	\vspace{-1mm}
\end{figure}

%{\color{red}
%} in~\cite{stern2019multi}, which is a 16x16 grids without obstacles. Fig.~\ref{fig:measures_vs_m_nr} shows the results, with fixed $N=2,4$ respectively, about (1) success rates of finding all Pareto-optimal solutions, (2) the average number of Pareto-optimal solutions found, (3) the average number of states expanded and (4) the average run time, as a function of $M$ (the number of objectives). {\color{red}Here, (2)(3) and (4) are averaged over solved test instances.}

%{\color{red}
We begin our tests with the a free grid of size 16 by 16 (``empty-16-16'') and the result is shown in Fig.~\ref{fig:measures_vs_m_nr}.
When $N=2$, from (a), both NAMOA$^*$ and MOM$^*$ computes all Pareto-optimal solutions for almost all the instances and outperforms MO-CBS.
The average number of Pareto-optimal solutions grows from one to $\approx$ one hundred when $M$ grows from one to three.
This trend indicates the difficulty of the problems with an increasing $M$: for the same graph with a fixed $N$, problems with a larger $M$ typically have more Pareto-optimal solutions and can take longer time for all three algorithms to compute.
Notice that when $M=1$, the MOMAPF problem becomes the conventional (single-objective) MAPF problem where an optimal solution is computed.
From (b), we observe a similar increase in both the average number of states expanded and run time over solved instances.
However, MOM$^*$ expands fewer states and runs faster than NAMOA$^*$ on average.
%}

%In addition, in (b), the gap between NAMOA* and MOM* become narrower as number of objectives increases.
%The possible reason is, when $M$ is large, the low dimensional search space constructed in MOM* is almost the same as the joint graph $\mathcal{G}$: for each agent $i\in I$, all neighbors of $v^i$ in $G$ are along some Pareto-optimal paths and the set of limited neighbors at $v$ is almost the same as the set of all neighbors in $G$.

%{\color{red}
When $N=4$, from (c), MOM$^*$ performs better than NAMOA$^*$ and similar to MO-CBS in terms of success rates.
Since MOM$^*$ solves more instances than NAMOA$^*$, the results in (d) becomes less informative.
But even when MOM$^*$ solves more (possibly hard) instances, the run time of MOM$^*$ is still obviously shorter than NAMOA$^*$.
%}

%This indicates that the burden of expanding a state is heavy in NAMOA*: NAMOA* expands a state $s$ by generating all neighbors of joint vertex $v(s)$ within joint graph $\mathcal{G}$ and the number of generated neighbors is called the branching factor.
%MOM* generates only limited neighbors indicated by Pareto policies and thus enjoys a smaller branching factor than NAMOA*, which makes MOM* computationally more efficient than NAMOA*.
 
%When $N \geq 6$, NAMOA* in general fails due to the enormous branching factor during the search (as shown in Table~\ref{tab:numerical_result} in the following section).
%We can summarize that, the performance gap between NOMOA* and MOM* is narrow if there is only a small number of agents (such as $N=2$) and several objectives (such as $M=4$).
%However, the performance gap between those two algorithms is large when there are only a few objectives (such as $M=2$) and a relatively large number of agents (such as $N \geq 4$). 
%More comparisons of MOM* and NAMOA* can be found in Table~\ref{tab:numerical_result}, which will be explained next.

\vspace{-0mm}
\subsection{Experiments with Different Number of Agents}
%{\color{red}
Fixing $M=2$ (two objectives), we evaluate all three algorithms in different maps in terms of (1) success rates of finding all Pareto-optimal solutions, (2) the average number of states expanded for NAMOA$^*$ and MOM$^*$, and (3) average number of solutions computed, with a varying number of agents $N$. The averages are taken over solved instances.
As shown in Table~\ref{tab:numerical_result}, MOM$^*$ (without heuristic inflation) outperforms NAMOA$^*$ as it achieves higher success rates and expands fewer states (when the success rates are similar).
Comparing with MO-CBS, MOM$^*$ (without heuristic inflation) performs better than MO-CBS in the room map and worse in ``den312d'' map.
In multi-objective settings, similar to the observation in \cite{felner2017search}, there is no single multi-agent planner that outperforms all other planners in all maps.
%}

\vspace{-0mm}
\subsection{Inflated MOM$^*$ in Different Maps}

\makeatletter
\newcommand\mytiny{\@setfontsize\notsotiny{6}{7}}
\makeatother
\newcommand{\mydata}[3]{{\small#1/ {\color{red} #2}/ {\color{blue} #3}}}

\begin{table*}[htbp]
    \vspace{4mm}
	\centering
%    \scriptsize
%    \small
	%	\mytiny
	\tabcolsep=0.07cm
	\renewcommand{\arraystretch}{1.13}
	\begin{tabular}{ | l | l | l | l | l | l | l | l | }
		\hline
		\multirow{2}{*}{Map} & \multirow{2}{*}{Alg.} & \multicolumn{6}{c}{Success rates / {\color{red} Avg. \#states expanded (in thousand)} / {\color{blue} Avg. \#solutions found}  } \vline \\ 
		& & N=2 & N=4 & N=8 & N=12 & N=16 & N=20 \\ \hline
		\multirow{6}{*}{\includegraphics[width=0.10\linewidth]{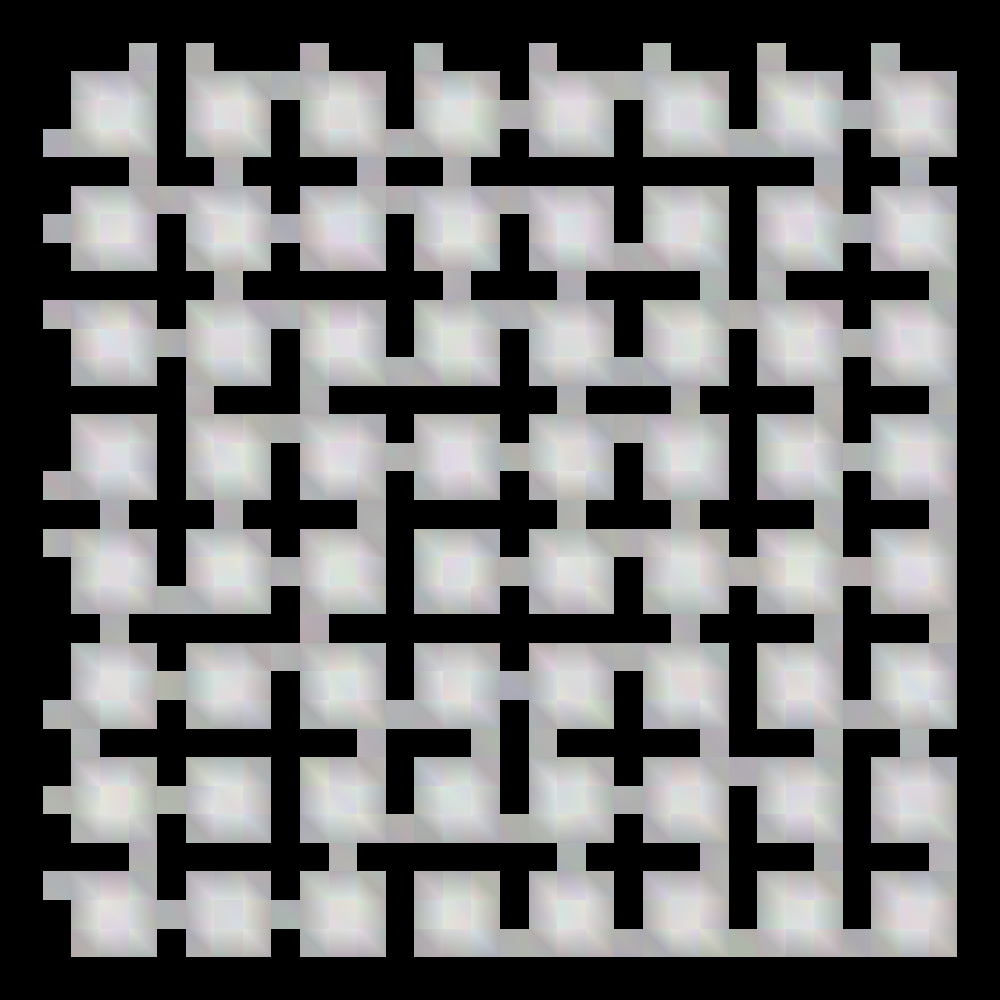}}& NAMOA$^*$ &\mydata{1.00}{3.3}{11.2} &\mydata{0.16}{78.1}{14.3} &\mydata{0}{-}{-} &\mydata{0}{-}{-} &\mydata{0}{-}{-} &\mydata{0}{-}{-} 
		\\ 
		& MO-CBS &\mydata{0.88}{-}{11.1}  &\mydata{0.48}{-}{25.0} &\mydata{0.08}{-}{72.5} & \mydata{0}{-}{-}  & \mydata{0}{-}{-}  & \mydata{0}{-}{-} 
		\\ 
		& MOM$^*$ &\mydata{1.00}{1.4}{11.2}  &\mydata{0.76}{36.0}{28.7} &\mydata{0.12}{65.6}{59.7} & \mydata{0}{-}{-}  & \mydata{0}{-}{-}  & \mydata{0}{-}{-} 
		\\ 
		& MOM$^*$ w=1.1 &\mydata{1.00}{0.1}{2.7} &\mydata{0.96}{1.5}{1.9} &\mydata{0.76}{46.1}{2.4} &\mydata{0.24}{5.8}{1.8} &\mydata{0.08}{140.6}{1.0}  & \mydata{0}{-}{-}  
		\\ 
		& MOM$^*$ w=1.2 &\mydata{1.00}{0.06}{1.5} &\mydata{0.96}{0.4}{1.3} &\mydata{0.80}{5.7}{1.6} &\mydata{0.48}{47.0}{1.4} &\mydata{0.24}{146.4}{1.5} & \mydata{0}{-}{-}  
		\\ 
		& MOM$^*$ w=1.5 &\mydata{1.00}{0.05}{1.0} &\mydata{1.00}{1.2}{1.0} &\mydata{0.96}{3.4}{1.1} &\mydata{0.84}{21.9}{1.1} &\mydata{0.56}{162.7}{1.1} &\mydata{0.04}{321.9}{1.0}
		\\ 
		(Room, 32x32)& MOM$^*$ w=2.0 &\mydata{1.00}{0.05}{1.0} &\mydata{1.00}{0.8}{1.0} &\mydata{1.00}{8.5}{1.1} &\mydata{0.96}{22.4}{1.1} &\mydata{0.68}{69.8}{1.0} &\mydata{0.16}{56.8}{1.0}
		\\ 	\hline
		\multirow{6}{*}{\includegraphics[width=0.10\linewidth]{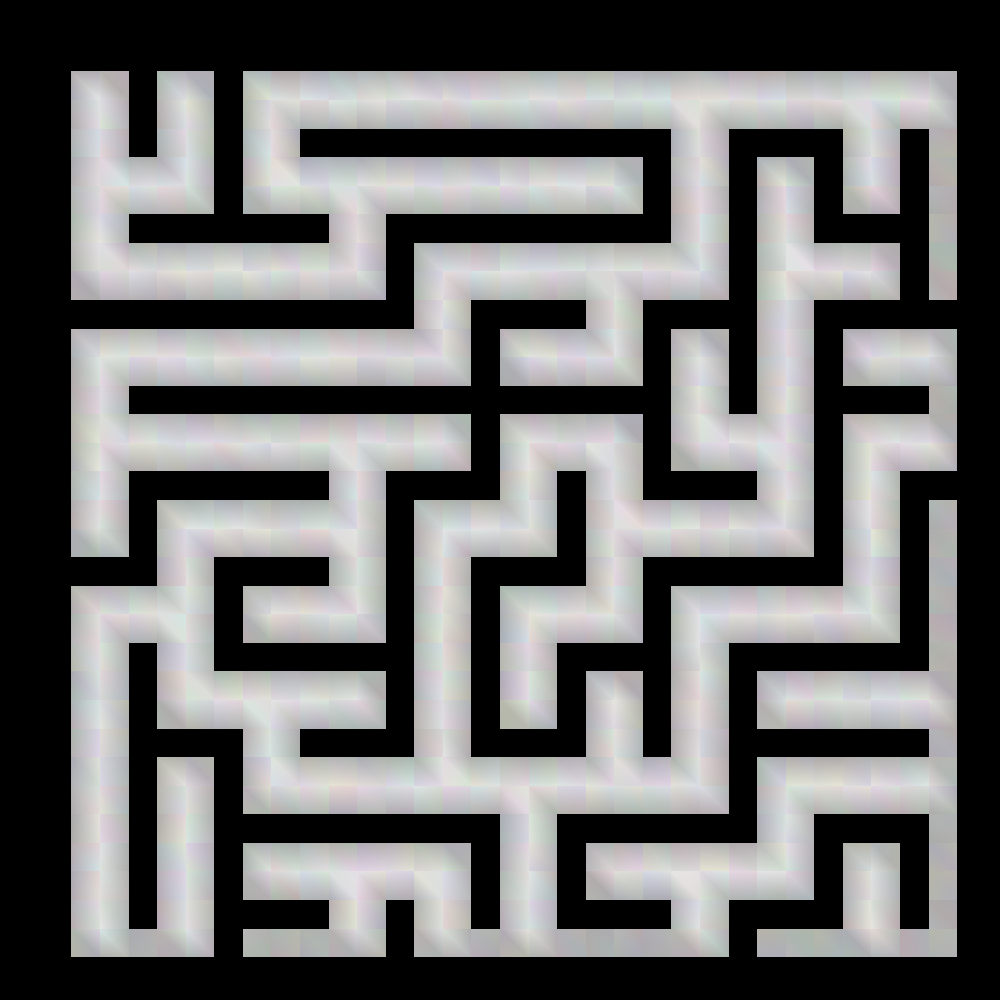}} & NAMOA$^*$ &\mydata{0.96}{40.6}{34.9} &\mydata{0}{-}{-} &\mydata{0}{-}{-} &\mydata{0}{-}{-} &\mydata{0}{-}{-} &\mydata{0}{-}{-} 
		\\ 
		& MO-CBS &\mydata{0.68}{-}{36.5}  &\mydata{0.20}{-}{128.2} &\mydata{0}{-}{-} & \mydata{0}{-}{-}  & \mydata{0}{-}{-}  & \mydata{0}{-}{-} 
		\\
		& MOM$^*$ &\mydata{0.96}{8.3}{34.9}  &\mydata{0.16}{123.1}{74.8} &\mydata{0}{-}{-}  &\mydata{0}{-}{-}  &\mydata{0}{-}{-}  &\mydata{0}{-}{-}
		\\ 
		& MOM$^*$ w=1.1 &\mydata{1.00}{0.4}{4.6}  &\mydata{1.00}{2.0}{5.0} &\mydata{0.72}{88.8}{3.3} &\mydata{0.16}{136.1}{2.0} &\mydata{0}{-}{-} &\mydata{0}{-}{-}  
		\\ 
		& MOM$^*$ w=1.2 &\mydata{1.00}{0.1}{1.9} &\mydata{1.00}{6.8}{2.1} &\mydata{0.84}{20.4}{1.8} &\mydata{0.36}{112.9}{1.6} &\mydata{0.08}{68.3}{1.0} & \mydata{0}{-}{-}  
		\\ 
		& MOM$^*$ w=1.5 &\mydata{1.00}{0.1}{1.2} &\mydata{1.00}{0.7}{1.3} &\mydata{1.00}{12.4}{1.2} &\mydata{0.92}{118.7}{1.3} &\mydata{0.16}{106.4}{1.3} & \mydata{0}{-}{-}  
		\\ 
		(Maze, 32x32)& MOM$^*$ w=2.0 &\mydata{1.00}{0.1}{1.1} &\mydata{1.00}{0.6}{1.2} &\mydata{1.00}{8.5}{1.0} &\mydata{0.96}{76.2}{1.1} &\mydata{0.40}{78.0}{1.1} &\mydata{0.04}{108.0}{1.0}
		\\ \hline
		\multirow{6}{*}{\includegraphics[width=0.10\linewidth]{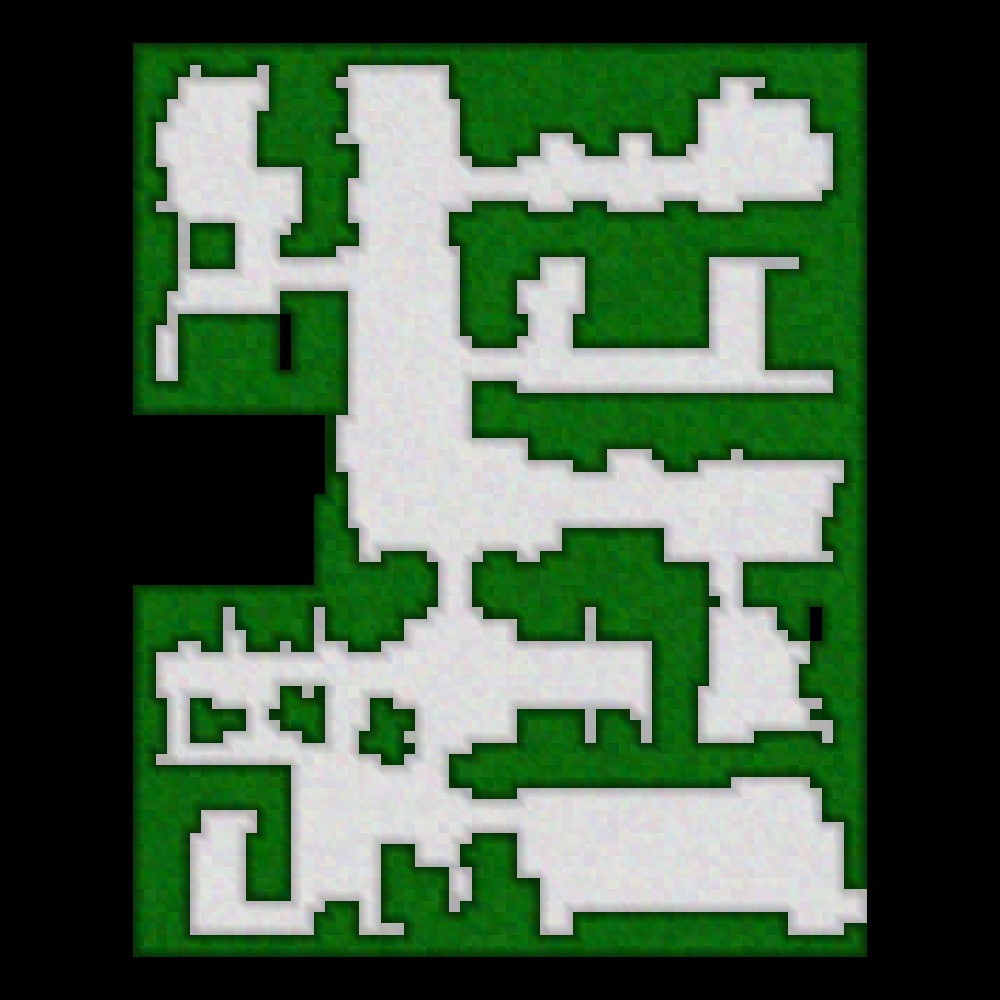}} & NAMOA$^*$ &\mydata{0.32}{121}{57.9} & \mydata{0}{-}{-} & \mydata{0}{-}{-}  & \mydata{0}{-}{-}  & \mydata{0}{-}{-} & \mydata{0}{-}{-}  
		\\ 
		& MO-CBS &\mydata{0.72}{-}{105.7}  &\mydata{0.12}{-}{180.3} &\mydata{0}{-}{-} & \mydata{0}{-}{-}  & \mydata{0}{-}{-}  & \mydata{0}{-}{-} 
		\\
		& MOM$^*$ &\mydata{0.56}{52.5}{88.9} &\mydata{0.04}{74.6}{106.0} & \mydata{0}{-}{-} & \mydata{0}{-}{-}  & \mydata{0}{-}{-}  & \mydata{0}{-}{-} 
		\\ 
		& MOM$^*$ w=1.1 &\mydata{1.00}{12.2}{18.7} &\mydata{0.88}{18.7}{9.0} & \mydata{0}{-}{-}  & \mydata{0}{-}{-}  & \mydata{0}{-}{-}  & \mydata{0}{-}{-}  
		\\ 
		& MOM$^*$ w=1.2 &\mydata{1.00}{0.6}{5.6} &\mydata{1.00}{7.6}{2.9} &\mydata{0.20}{12.2}{3.2} &\mydata{0}{-}{-} &\mydata{0}{-}{-} &\mydata{0}{-}{-}  
		\\ 
		& MOM$^*$ w=1.5 &\mydata{1.00}{0.2}{1.5} &\mydata{1.00}{0.4}{1.2} &\mydata{0.24}{5.5}{1.2} &\mydata{0}{-}{-} &\mydata{0}{-}{-} &\mydata{0}{-}{-}  
		\\ 
		(den312d, 65x81)& MOM$^*$ w=2.0 & \mydata{1.00}{0.2}{1.2}  &\mydata{1.00}{0.3}{1.1} &\mydata{0.28}{5.0}{1.1} &\mydata{0}{-}{-}  &\mydata{0}{-}{-}  &\mydata{0}{-}{-} 
		\\ \hline
	\end{tabular}
	\caption{Numerical results in various environments with two objectives ($M=2$).}
	\label{tab:numerical_result}
	\vspace{-4mm}
\end{table*}

To investigate the impact of inflated heuristics of MOM$^*$ as well as how inflated MOM$^*$ approximates Pareto-optimal set, Table~\ref{tab:numerical_result} shows the numerical results of MOM$^*$ with different inflation rates in different maps.

From Table~\ref{tab:numerical_result}, inflated heuristics enable MOM$^*$ to improve the computational efficiency by approximating Pareto-optimal sets.
First, for the same $N$ (the same column in the table) and similar success rates, as $w$ increases, MOM$^*$ expands fewer states and thus terminates earlier. Second, MOM$^*$ with larger $w$ finds fewer (bounded sub-optimal) solutions. 
Finally, MOM$^*$ is able to solve more instances with larger $N$ while MOM$^*$ with smaller or no heuristic inflation fails.

%In addition, Table~\ref{tab:numerical_result} shows that in different maps, MOM*, with or without inflation, can behave differently. We can roughly say that map den312d (65x81) is the most different map and Maze (32x32) is a bit more difficult than Room (32x32), based on the success rates of (inflated) MOM*. How map structures influence the algorithm performances remains to be an open question.

\vspace{-0mm}
\subsection{Time for Computing Pareto Policies}

Finally, we investigate the required effort for the pre-processing step, where heuristics and Pareto policies are computed.
%Although this part is closely related to (single-agent) multi-objective heuristic search literature, we still discuss it since Pareto policies are important building blocks of the proposed MOM* algorithm. 
Table~\ref{table:policy_time_cdims} shows the computational time needed with $M=1,2,3,4$ in a 16x16 empty map.
When $M=1$, Pareto policies for agents becomes the conventional optimal policies for agents as used in M$^*$~\cite{wagner2015subdimensional} since there is only one optimal neighbor for agent $i \in I$ to choose at any vertices along its path to its goal.
In this case, the efforts for pre-processing is equivalent to an exhaustive backwards A$^*$ (or Dijkstra) search over the graph for each agent, which is computationally cheap (0.037 seconds) from our results.
However, as $M$ increases, it takes significantly more time to compute Pareto policies since exhaustive multi-objective search over a graph is much more burdensome (114.07 seconds when $M=4$).

Table~\ref{table:policy_time_maps} shows the required time for pre-processing with two objectives ($M=2$ fixed) in different maps. 
From the table, the size of the map is a decisive factor. 
However, maps with the same size but of different types may also have obviously different computing time for Pareto policies, which is worthwhile for further exploration.

\begin{table}[htbp]
	\centering
	\vspace{-2mm}
	\begin{tabular}{ |c|c|c|c|c| } 
		\hline
		M & 1 & 2 & 3 & 4  \\
		\hline
		Time (seconds) & 0.037 & 0.46 & 5.65 & 114.07 \\ 
		\hline
	\end{tabular}
%	\vspace{-3mm}
	\caption{Pre-processing time as a function of $M$ with $N=2$.}
	\label{table:policy_time_cdims}
	\vspace{-2mm}
\end{table}

\begin{table}[htbp]
	\centering
	\vspace{-2mm}
	\begin{tabular}{ |c|c| } 
		\hline
		Map (Width x Height) & Time (seconds)  \\
		\hline
		Empty (16x16) & 0.46 \\ 
		\hline
		Room (32x32) & 0.92 \\ 
		\hline
		Maze (32x32) & 4.99 \\
		\hline
		den312d (65x81) & 76.88 \\
		\hline
	\end{tabular}
%	\vspace{4mm}
	\caption{Pre-processing time when $M=2$, $N=2$ in different maps.}
	\label{table:policy_time_maps}
	\vspace{-4mm}
\end{table}

\section{Conclusions and Future Work}\label{sec.conclud}

A new algorithm called MOM$^*$ and its variant inflated MOM$^*$ were presented for a multi-objective, multi-agent problem. We prove that MOM$^*$ is complete and finds all cost-unique Pareto-optimal solutions. 
Numerical results were also presented to compare the performance of MOM$^*$ with the existing methods.

There are several directions for future work in this area.
The current approach considers additive costs across agents and one can extend the method to other types of cost such as makespan.
This work considers minimization type of problem and one can explore how to apply the proposed algorithm to maximization type of problem such as multi-objective information gathering~\cite{chen2019pareto}.
One can also focus on extending other MAPF algorithms from single-objective to multi-objectives such as EPEA$^*$~\cite{goldenberg2014enhanced}, SAT-based methods~\cite{surynek2017sat}.
Another direction is about approximating the Pareto-optimal set via different approaches other than inflated heuristics, such as ideas from multi-objective evolutionary algorithms.

\bibliographystyle{plain}
\bibliography{bib_iros21}
% \bibliography{../contents/references}

%\newpage
%\section*{Appendix}
%\input{../contents/appendix}

\end{document}